\documentclass[accepted]{article}

\usepackage{microtype}
\usepackage{graphicx}
\usepackage{subfigure}
\usepackage{booktabs} % for professional tables

\usepackage{hyperref}
\usepackage{multirow}
\usepackage{tabularx}

\usepackage{icml2024}

\usepackage{amsmath}
\usepackage{amssymb}
\usepackage{mathtools}
\usepackage{amsthm}

\usepackage[capitalize,noabbrev]{cleveref}

\theoremstyle{plain}
\newtheorem{theorem}{Theorem}[section]
\newtheorem{proposition}[theorem]{Proposition}
\newtheorem{lemma}[theorem]{Lemma}
\newtheorem{corollary}[theorem]{Corollary}
\theoremstyle{definition}
\newtheorem{definition}[theorem]{Definition}
\newtheorem{assumption}[theorem]{Assumption}
\theoremstyle{remark}
\newtheorem{remark}[theorem]{Remark}

\usepackage[textsize=tiny]{todonotes}
\icmltitlerunning{Adapprox: Adaptive Approximation in Adam Optimization via Randomized Low-Rank Matrices}

\begin{document}

\twocolumn[
\icmltitle{Adapprox: Adaptive Approximation in Adam Optimization via Randomized Low-Rank Matrices}

% It is OKAY to include author information, even for blind
% submissions: the style file will automatically remove it for you
% unless you've provided the [accepted] option to the icml2024
% package.

% List of affiliations: The first argument should be a (short)
% identifier you will use later to specify author affiliations
% Academic affiliations should list Department, University, City, Region, Country
% Industry affiliations should list Company, City, Region, Country

% You can specify symbols, otherwise they are numbered in order.
% Ideally, you should not use this facility. Affiliations will be numbered
% in order of appearance and this is the preferred way.
% \icmlsetsymbol{equal}{*}

\begin{icmlauthorlist}
\icmlauthor{Pengxiang Zhao}{hku}
\icmlauthor{Ping Li}{hw}
\icmlauthor{Yingjie Gu}{hw}
\icmlauthor{Yi Zheng}{hw}
\icmlauthor{Stephan Ludger Kölker}{hw}
\icmlauthor{Zhefeng Wang}{hw}
\icmlauthor{Xiaoming Yuan}{hku}
\end{icmlauthorlist}

\icmlaffiliation{hku}{Department of Mathematics, The University of Hong Kong, Hong Kong, China}
\icmlaffiliation{hw}{System AI Innovation Lab, Huawei Cloud, Hangzhou, China}

\icmlcorrespondingauthor{Xiaoming Yuan}{xmyuan@hku.hk}

% You may provide any keywords that you
% find helpful for describing your paper; these are used to populate
% the "keywords" metadata in the PDF but will not be shown in the document
\icmlkeywords{Machine Learning, Optimizers, Memory Consumption, Matrix Factorization Techniques}

\vskip 0.3in
]

% this must go after the closing bracket ] following \twocolumn[ ...

% This command actually creates the footnote in the first column
% listing the affiliations and the copyright notice.
% The command takes one argument, which is text to display at the start of the footnote.
% The \icmlEqualContribution command is standard text for equal contribution.
% Remove it (just {}) if you do not need this facility.

\printAffiliationsAndNotice{Work was done when Pengxiang Zhao was an intern at Huawei Cloud System AI innovation Lab.}  % leave blank if no need to mention equal contribution
% \printAffiliationsAndNotice{\icmlEqualContribution} % otherwise use the standard text.

\begin{abstract}
As deep learning models exponentially increase in size, optimizers such as Adam encounter significant memory consumption challenges due to the storage of first and second moment data. 
Current memory-efficient methods like Adafactor and CAME often compromise accuracy with their matrix factorization techniques. 
Addressing this, we introduce Adapprox, a novel approach that employs randomized low-rank matrix approximation for a more effective and accurate approximation of Adam's second moment. 
Adapprox features an adaptive rank selection mechanism, finely balancing accuracy and memory efficiency, and includes an optional cosine similarity guidance strategy to enhance stability and expedite convergence. 
% Our empirical evaluations, focusing on training GPT-2 models and subsequent downstream tasks, show that Adapprox not only achieves memory savings of 33.8\% to 49.9\% over Adam while retaining the first moment (and up to 99.9\% when omitted) but also enhances convergence speed and overall performance on pretraining and downstrem tasks compared to counterparts.
% In GPT-2 training and subsequent downstream tasks, Adapprox achieves 34.5\%-49.9\% memory savings for the 117M model and 33.8\%-49.9\% for the 345M model with the first moment enabled compared to AdamW (without it, savings rise to 84.5\%-99.9\% and 83.8\%-99.9\%), respectively, also enhancing convergence speed and task performance compared to counterparts.
In GPT-2 training and downstream tasks, Adapprox surpasses AdamW by achieving 34.5\% to 49.9\% and 33.8\% to 49.9\% memory savings for the 117M and 345M models, respectively, with the first moment enabled, and further increases these savings without the first moment. Besides, it enhances convergence speed and improves downstream task performance relative to its counterparts.
\end{abstract}

\section{Introduction}
In the field of deep learning, optimization algorithms play a pivotal role in training models both efficiently and effectively. Among the most popular optimization algorithms is the Adam \cite{kingma2014adam}, and its variant, AdamW \cite{loshchilov2018decoupled}, known for their robust performance across diverse applications. However, the shift from smaller architectures like AlexNet \cite{krizhevsky2017imagenet} with fewer than 100 million parameters, to colossal models such as GPT-3 \cite{brown2020language}, encompassing over 100 billion parameters, poses substantial memory consumption challenges. This issue becomes particularly acute in resource-constrained environments \cite{steiner2023model}. Despite its effectiveness, the Adam optimizer exacerbates this challenge by necessitating significant memory to store both first and second moments for each parameter to maintain adaptive learning rates.

\begin{figure}[t]
\vskip 0.2in
\begin{center}
\centerline{\includegraphics[width=0.48\textwidth]{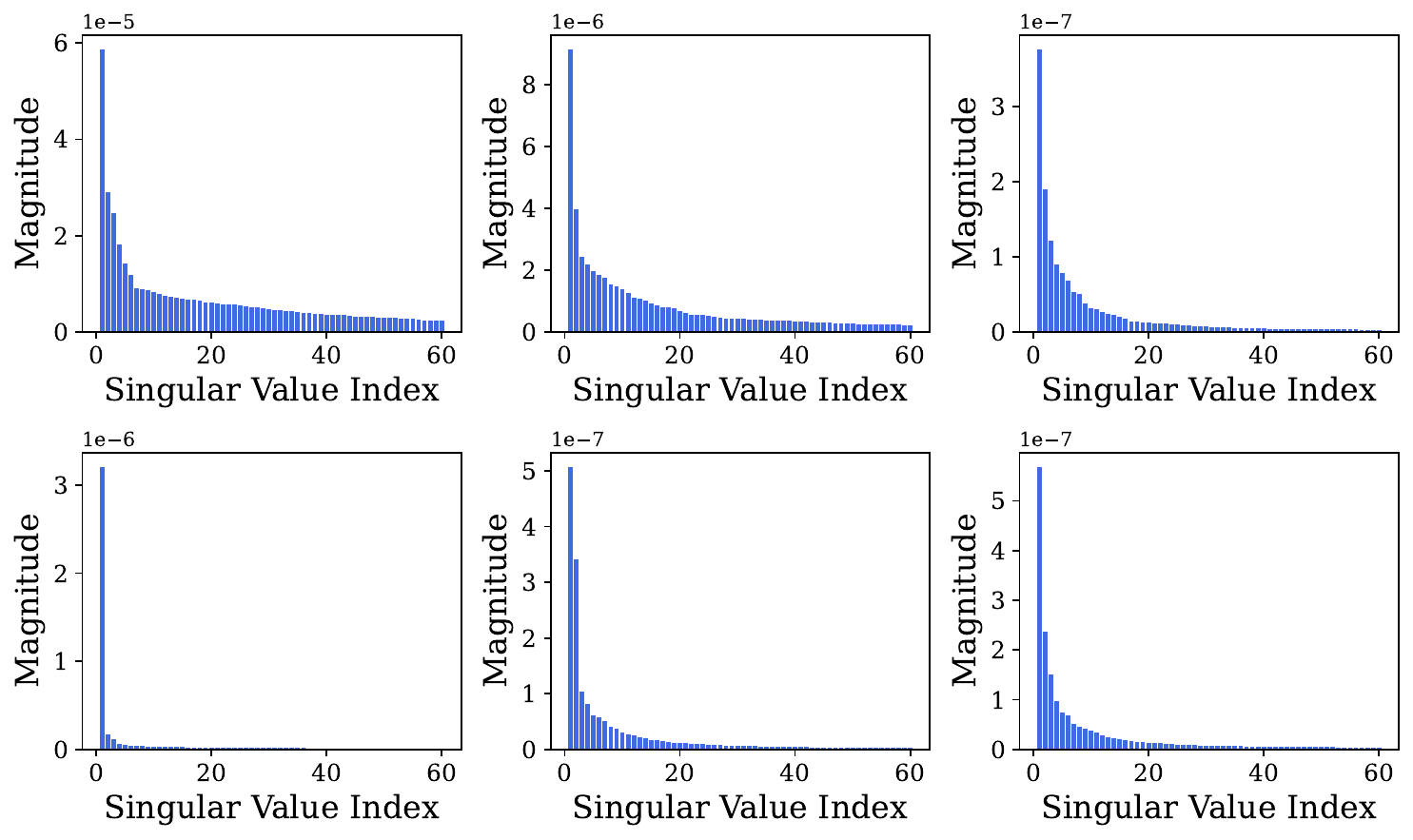}}
\caption{Singular value distributions. This figure shows the top 60 singular values from six second moment matrices, out of a full rank of 1,024, obtained from AdamW training a GPT-2 345M model at the 45,000th iteration.
%The $x$-axis represents the singular value index, while the $y$-axis indicates their respective magnitudes. 
%Each subplot corresponds to a different matrix, highlighting the prevalence of dominant singular values.
}
\label{fig:svd}
\end{center}
\vskip -0.2in
\end{figure}

%Brief literature review on Adam optimization and existing approximation methods.
% Therefore, optimizing memory usage in Adam is critical and serves as an important complementary mechanism alongside other memory reduction strategies such as quantization \cite{han2015deep} and model compression \cite{hinton2015distilling}.
% , and recomputation \cite{chen2016training}. 
There have been notable efforts in the development of memory-efficient optimizers \cite{shazeer2018adafactor, anil2019memory, li2023memory, luo2023came}. %Among these advancements, 
Adafactor \cite{shazeer2018adafactor} 
% stands out with its innovative adaptation of the Adam optimizer. It uniquely
offers the option to omit Adam's first moment, while utilizing a factored representation for the second moment.
% which demonstrates a substantially higher compression rate compared to methods that utilize quantization in optimizer states \cite{li2023memory}. 
% Adafactor optionally eliminates the need for storing the entire first moment and transforms the memory requirements for an $mn$-element second moment matrix from $O(mn)$ to $O(m+n)$, which demonstrates a substantially higher compression rate compared to methods that utilize quantization in optimizer states \cite{li2023memory}. 
However, Adafactor tends to show a decrease in training effectiveness, which is primarily linked to the unavoidable approximation errors that occur in its matrix factorization approach \cite{anil2019memory, luo2023came}. 
To mitigate this performance degradation and improve training stability, the CAME optimizer \cite{luo2023came} extends Adafactor with a confidence-based scaling factor for the parameter updates.
Nevertheless, CAME still relies on the same fundamental factorization technique to store the second moment and confidence statistics, thus inheriting the core challenges encountered by Adafactor.

% \begin{figure*}[t]
% \vskip 0.2in
% \begin{center}
% \centerline{\includegraphics[width=\textwidth]{./figs/paper_svd.pdf}}
% \caption{Singular value distributions. This figure presents the distributions of the top 60 singular values (out of a possible full rank of 1,024) for eight representative second moment matrices extracted from AdamW training a GPT-2 345M model at the 45,000th iteration. The $x$-axis represents the singular value index, while the $y$-axis indicates their respective magnitudes. 
% %Each subplot corresponds to a different matrix, highlighting the prevalence of dominant singular values.
% }
% \label{fig:svd}
% \end{center}
% \vskip -0.2in
% \end{figure*}

In this study, we present Adapprox, a novel approach designed to overcome the limitations associated with memory-efficient optimizers that rely on matrix factorization. Our work is motivated by a crucial observation: in many large-scale model training scenarios, the second moment matrices often have a limited number of dominant singular values, with the rest exhibiting substantially lower magnitudes. This insight is corroborated by empirical evidence, illustrated in Figure \ref{fig:svd}.
% , which reveals that the second moment matrices are distinguished by a small number of considerably higher dominant singular values, with the rest exhibiting substantially lower magnitudes. 
Consequently, Adapprox employs randomized low-rank matrix approximation \cite{liberty2007randomized, li2014large, batselier2018computing} to effectively approximate the second moment in Adam. Our method primarily reduces memory usage by distilling key features from large matrices, while also ensuring a more precise representation.

% This figure displays the singular value distributions of eight representative second moment matrices from the AdamW \cite{loshchilov2018decoupled} optimization of the GPT-2 345M model \cite{radford2019language}. In the figure, the $x$-axis represents the singular value index, while the $y$-axis indicates their respective magnitudes. 

%Compared to standard methods like full Singular Value Decomposition (SVD) or iterative techniques \cite{candes2012exact, lee2013local}, randomized low-rank matrix approximation provides a computationally efficient solution for handling large-scale matrices \cite{liberty2007randomized, li2014large, batselier2018computing}.

Figure \ref{fig:svd} also elucidates the limitations of fixed 1-rank approximations used in Adafactor and CAME. This approach's tendency to compromise accuracy due to multiple dominant singular values is evident in the top row of plots in Figure~\ref{fig:svd}. However, in these instances, a modest increase in the target rank for approximation can substantially improve accuracy.
% a substantial improvement in accuracy can be achieved by making a modest increase in the target rank for approximation. 
% Typically, overestimating the target rank in low-rank approximations results in computationally intensive processes with marginal improvements in precision. Conversely, underestimation can significantly compromise the accuracy. In the case of Adapprox, we implement a dynamic rank-selection mechanism, which heuristically selects an appropriate rank for each target matrix in low-rank approximation.
Considering the singular value distribution patterns of the second moment, as illustrated in Figure \ref{fig:svd}, overestimating the target rank in low-rank approximations can lead to computationally intensive processes with minimal precision improvements. Conversely, underestimating the rank significantly compromises accuracy. In response to this, Adapprox introduces a dynamic rank-selection mechanism, designed to adaptively select an optimal rank for each target matrix within the low-rank approximation process.

% Our proposed method not only improves the precision of approximations but also effectively addresses the issues arising from arbitrary rank selection. Arbitrary selection often leads to overestimation, which, in turn, results in computationally intensive processes with only slight improvements in precision. Alternatively, underestimation due to arbitrary rank selection compromises both the accuracy and the overall performance of the algorithm.

Furthermore, to mitigate systematic errors inherent in low-rank approximations, our approach optionally integrates a cosine similarity guidance mechanism. This method computes the cosine similarity between each update and the first moment, subsequently adjusting the learning rate for that step proportionally. A low cosine similarity results in a reduced update, while a high similarity prompts an increase in the update magnitude.

We evaluated our method through the pretraining of GPT-2 across various configurations and associated downstream tasks. The findings demonstrate that our approach achieves reduced memory usage compared to Adam, while maintaining only a marginally higher memory footprint than Adafactor and CAME. Crucially, our method not only maintains performance on par with Adam but also enables accelerated convergence and improved generalization performance.

Our study's key contributions are as follows:
\begin{itemize}
    \item In light of the singular value distribution of the second moment in Adam, we present Adapprox, which utilizes randomized low-rank matrix approximation to effectively approximate the second moment in Adam.
    \item We develop an adaptive rank selection mechanism that balances precision and memory savings by choosing an appropriate rank for approximating the target matrix.
    \item We integrate an optional cosine similarity guidance strategy into Adapprox, which aims to expedite the convergence process and enhance the stability.
    \item We showcase Adapprox's efficacy on the GPT-2 pretraining and several downstream tasks. 
    Enabling the first moment, Adapprox achieves memory savings ranging from 34.5\% to 49.9\% for the GPT-2 117M model and 33.8\% to 49.9\% for the 345M model, compared to AdamW.  
    Disabling the first moment elevates these savings to 84.5\% to 99.9\% for the 117M model, and to 83.8\% to 99.9\% for the 345M model.
    Furthermore, our experimental results suggest that Adapprox may offer faster convergence and potentially improved performance outcomes.
\end{itemize}

\section{Related Work}
\textbf{Memory Efficient Optimizers.}
Memory efficient optimizers aim to reduce memory usage by compressing optimizer states during training, ideally without affecting the performance efficacy of their standard counterparts.
% Adafactor \cite{shazeer2018adafactor} is an evolution from Adam, achieving memory reduction through two principal strategies. Firstly, it offers the option to omit the first moment. Secondly, it compresses the second moment via an innovative optimal 1-rank matrix factorization method grounded in I-divergence \cite{lee1999learning} minimization. 
Adafactor \cite{shazeer2018adafactor} reduces memory usage by employing two main strategies: the first is the optional omission of Adam's first moment; and the second is the compression of the second moment using a novel 1-rank matrix factorization approach based on minimizing I-divergence \cite{lee1999learning}.
% Although Adafactor demonstrates significant memory savings and computational efficiency, it is important to acknowledge the potential trade-off of reduced accuracy due to such factorization \cite{anil2019memory, luo2023came}.
While Adafactor achieves notable memory savings and computational efficiency, it is crucial to recognize the potential compromise in accuracy that may arise from its factorization approach \cite{anil2019memory, luo2023came}.
SM3 \cite{anil2019memory} represents a memory efficient variation of Adagrad \cite{duchi2011adaptive}. Experimentally, it proves especially effective in situations where gradients show natural activation patterns. 
CAME \cite{luo2023came} builds upon Adafactor's framework by incorporating a confidence-guided strategy to elevate approximation accuracy. However, this method also involves compressing and storing confidence data using an identical matrix factorization technique, resulting in similar accuracy-related challenges. Moreover, CAME's dependency on the first moment for its confidence strategy limits the potential for omitting this data, thereby constraining further memory savings.
4-bit Adam \cite{li2023memory} employs quantization techniques to compress Adam's first and second moments. Notably, this quantization is compatible with matrix factorization methods.

\textbf{Low-Rank Matrix Approximation.} Low-rank matrix approximation seeks to represent a matrix using lower-rank matrices, aiming for a more efficient data representation while endeavoring to retain as much information as possible.
% Consider a target matrix $A \in \mathbb{R}^{m \times n}$; its $k$-rank approximation can be expressed as $A \approx QU^\top$, where $Q \in \mathbb{R}^{m \times k}$ and $U \in \mathbb{R}^{n \times k}$ are two feature matrices ($1 \leq k \ll \min\{m, n\}$). This approach allows for more economical storage, reducing the space complexity from $O(mn)$ to $O(k(m+n))$. 
Low-rank matrix approximation is utilized across a wide spectrum of applications, including principal component analysis (PCA) \cite{shen2008sparse, papailiopoulos2013sparse}, image processing \cite{haeffele2014structured, guo2017patch, chen2017denoising}, and a variety of machine learning scenarios \cite{paterek2007improving, li2016low}.

\section{Methodology}
% In this section, we begin with a brief overview of Adam, followed by an exploration of the randomized low-rank approximation method. We then present our proposed adaptive rank selection mechanism to balance memory efficiency and performance. Next, we provide a detailed description of the Adapprox optimizer. Additionally, we discuss a cosine similarity guidance strategy, designed to enhance stability and expedite convergence.
This section begins with a brief overview of Adam, followed by an examination of the randomized low-rank approximation method. Subsequently, we introduce the adaptive rank selection mechanism. We then provide a comprehensive description of the Adapprox optimizer. Lastly, we delve into the proposed cosine similarity guidance strategy.
% , designed to enhance stability and hasten convergence.

\subsection{Overview of the Adam Optimizer}
Consider a function $f(W)$, where $W \in \mathbb{R}^{m \times n}$ denotes the parameters of the neural network.
% (such as in a linear layer, or attention matrices). 
The update rule for Adam \cite{kingma2014adam} at the $t$-th iteration is defined as follows:
\begin{align}
\text{(Adam)}
\left\{
\begin{array}{ll}
G_t & = \nabla f(W_{t-1}), \\
M_t & = \beta_1 M_{t-1} + (1 - \beta_1) G_t,  \\
V_t & = \beta_2 V_{t-1} + (1 - \beta_2) G_t^2, \\
\hat{M}_t & = M_t/(1 - \beta_1^t), \\
\hat{V}_t & = V_t/(1 - \beta_2^t), \\
W_{t} & = W_{t-1} - \alpha \hat{M}_t/(\sqrt{\hat{V}_t} + \epsilon).
\end{array}
\right.
\end{align}
Here, all computations are element-wise. $G_t$ represents the gradient arranged in matrix form. $M_t$ and $V_t$ are the exponential running averages of the first and second moments, respectively. $\hat{M}_t$ and $\hat{V}_t$ are the bias-corrected versions of $M_t$ and $V_t$. $\beta_1$ and $\beta_2$ control these moment estimates. Additionally, $\alpha$ denotes the learning rate, and $\epsilon$ is a small positive constant introduced to prevent division by zero.

Building upon Adam, AdamW \cite{loshchilov2018decoupled} decouples weight decay from the gradient updates. With this change, the parameter update step in AdamW is
\begin{equation}\label{equ:weight_decay}
    W_{t} = W_{t-1} - \alpha( \hat{M}_t/(\sqrt{\hat{V}_t} + \epsilon) + \lambda W_{t-1}),
\end{equation}
where $\lambda$ is the rate of the weight decay.

As mentioned, Adam requires the storage of both $M_t$ and $V_t$ at each step, which requires $O(mn)$ extra memory.
% Consequently, compressing these components becomes essential to enhance memory efficiency \cite{shazeer2018adafactor, luo2023came}.

\subsection{Low-Rank Approximation of the Second Moment}
For a matrix $A \in \mathbb{R}^{m \times n}$, deriving its low-rank approximation can be formulated as an optimization problem:
\begin{equation}
    \min_{Q, U} \|A - QU^\top\|_F^2,
\end{equation}
where $\|\cdot\|_F$ is the Frobenius norm and $Q \in \mathbb{R}^{m \times k}$ and $U \in \mathbb{R}^{n \times k}$ are two feature matrices ($1 \leq k \ll \min\{m, n\}$). Then, $A_k = QU^\top$ is the $k$-rank approximation of $A$.
The optimal $A_k$ can be determined by performing a full Singular Value Decomposition (SVD) and then truncating it to retain only the top 
$k$ singular values and their corresponding singular vectors, resulting in the following representation (see Theorem 2.4.8 in \cite{golub2013matrix}):
\begin{equation}
    A_k = \sum_{i=1}^k \sigma_i u_i v_i^T,
\end{equation}
where $\sigma_1 \geq \sigma_2 \geq \dots \geq \sigma_k \geq 0$ are the top $k$ singular values of $A$, and $u_i$ and $v_i$ $(1\leq i \leq k)$ are corresponding left and right singular vectors.
The approximation error can be explicitly expressed as follows:
\begin{equation}\label{equ:svd_error}
    \|A-A_k\|_F^2 = \sum_{i=k+1}^{\min\{m, n\}}\sigma_i^2.
\end{equation}
% As depicted in Figure \ref{fig:svd}, the distribution of singular values for the second moment matrix in training large language models, such as GPT-2, typically exhibits a few dominant singular values accompanied by numerous smaller values. Consequently, employing $k$-rank approximation to compress the second moment matrix becomes a logical approach in this context.
Given Equation \eqref{equ:svd_error} and the singular value distribution of the second moment matrix shown in Figure \ref{fig:svd}, utilizing a $k$-rank approximation to compress the second moment matrix emerges as a rational choice.
% Unlike Adafactor's fixed rank-1 approximation, our approach opts for a more adaptable low-rank approximation method. This is particularly beneficial in instances where the second moment matrices display multiple dominant singular values that cannot be overlooked. Consequently, employing a slightly higher rank for approximation can lead to substantial improvements in precision.
Nevertheless, the computation of the full SVD for large matrices presents considerable computational and memory challenges. We mitigate these issues by leveraging randomized low-rank matrix approximation algorithms \cite{liberty2007randomized, halko2011finding, nakatsukasa2020fast}, which provide a balance of computational efficiency and memory economy, while still delivering high-quality low-rank approximations. 

Our implementation utilizes the Gaussian sampling variant of the randomized SVD algorithm \cite{halko2011finding}. In our approach, we bypass the SVD estimation to streamline the process, concentrating solely on the extraction of feature matrices without the need for singular values. Additionally, we incorporate an oversampling mechanism, which further refines the algorithm’s ability to capture more precise subspace representations. The comprehensive procedure of this modified method is detailed in Algorithm \ref{alg:rsi}, termed Streamlined Randomized Subspace Iteration (S-RSI).

\begin{algorithm}[t]
\caption{Streamlined Randomized Subspace Iteration}
\label{alg:rsi}
\begin{algorithmic}
\STATE \textbf{Inputs:} Target matrix $A \in \mathbb{R}^{m \times n}$, target rank $k$, integer $l$, and integer $p$ with $(k+p) \leq \min\{m, n\}$
%\STATE \text{Draw an $n\times (k+p)$ standard Gaussian matrix $U$}
\STATE $U \sim \mathcal{N}(0, 1)$
\STATE $Q, R \leftarrow \mathbf{0}^{m \times (k+p)}, \mathbf{0}^{(k+p) \times (k+p)}$
\FOR{$i \leftarrow 1, 2, \dots, l$}
    \STATE $Q \leftarrow AU$
    \STATE $Q, R \leftarrow \text{QR decomposition}(Q)$
    \STATE $U \leftarrow A^\top Q$
\ENDFOR
\STATE \textbf{return} $Q[:, :k], U^\top[:,:k]$
\end{algorithmic}
\end{algorithm}

The S-RSI aims to compute an approximate basis $Q \in \mathbb{R}^{m\times k}$ which has orthonormal columns for the column space of the target matrix $A \in \mathbb{R}^{m \times n}$ such that
\begin{equation}
    A_k = Q Q^\top A.
\end{equation}
We form $U = Q^\top A$, and thus we obtain two feature matrices $Q \in \mathbb{R}^{m\times k}$ and $U^\top \in \mathbb{R}^{n \times k}$. 
% The task of computing $Q$ is executed very efficiently with random sampling methods.
% Suppose that we draw a random vector $u$, ensuring that the elements of $u$ are independently and identically distributed according to a standard Gaussian distribution.  Upon computing $q = Au$, we posit that $q$ serves as a stochastic representation from the column space of $A$. This assertion is rooted in the fact that $q$ is inherently a random linear combination of columns of $A$.
Computing $Q$ is highly efficient using random sampling methods. Consider drawing a random vector $u$, with each element independently and identically distributed according to a standard Gaussian distribution. When we compute $q = Au$, it effectively serves as a stochastic representation of the column space of $A$, because $q$ represents a random linear combination of the columns of $A$.
By repeating this sampling process $k$ times, we get a set of random vectors:
\begin{equation}\label{equ:q_def}
    \{q_i\ |\ q_i = Au_i,\ i=1, 2, \dots, k\}.
\end{equation}
Due to the inherent randomness in $u_i$'s generation, the set of vectors $\{u_i\}_{i=1}^k$ are expected to occupy a general linear position, which implies a high likelihood that any subset of these vectors is linearly independent. Consequently, this observation leads us to propose the following:
\begin{proposition}
Given a set of randomly generated vectors $\{u_i\}_{i=1}^k$ that are in a general linear position, and a full rank matrix $A \in \mathbb{R}^{m \times n}$, the set of vectors $\{q_i\ |\ q_i = Au_i\}_{i=1}^k$ are also linearly independent.
\end{proposition}
\begin{proof} 
% The assertion relies on the fundamental property of linear independence among the vectors $\{u_i\}_{i=1}^k$ and the full rank characteristic of $A$.
Linear independence of $\{u_i\}_{i=1}^k$ implies that $\sum_{i=1}^k a_i u_i = 0$ holds when all scalars $\{a_i\}_{i=1}^k$ are zero. We examine a linear combination of the vectors $\{q_i\}_{i=1}^k$:
\begin{equation}
   \sum_{i=1}^k a_i q_i = \sum_{i=1}^k a_i A u_i =  A\sum_{i=1}^k a_i u_i.
\end{equation}
As $A$ is full rank, $\sum_{i=1}^k a_i q_i \neq 0$ unless all $a_i$ are zero, indicating that vectors $\{q_i\}_{i=1}^k$ are linearly independent.
\end{proof}
% Consequently, the path to deriving an orthonormal basis for the column space of matrix $A$ is paved by initially arranging the set $\{q_i\}_{i=1}^k$ as the columns of a matrix $Q$ and then applying an appropriate orthonormalization procedure, such as QR decomposition \cite{golub2013matrix}.
Consequently, to derive an orthonormal basis for the column space of matrix $A$, we first arrange the set $\{q_i\}_{i=1}^k$ as the columns of a matrix $Q$, and then we apply an orthonormalization procedure, like QR decomposition \cite{golub2013matrix}.

To enhance the precision of random sampling for basis determination in scenarios where the input matrix exhibits a flat singular spectrum and is of considerable size, we incorporate a power iteration technique \cite{rokhlin2010randomized, halko2011finding, golub2013matrix}. Specifically, this approach applies the randomized sampling methodology to the modified matrix
\begin{equation}
    A^\prime = (A A^\top)^l A,
\end{equation}
where $l$ is a small integer (for example, $l=5$). 
According to SVD, matrix $A$ can be expressed as $A = Q\Sigma U^T$, where $Q$ and $U$ are orthogonal matrices encapsulating the left and right singular vectors of $A$, respectively, and $\Sigma$ is a diagonal matrix composed of the singular values of $A$. 
% Leveraging this decomposition, matrix $A^\prime$ is derived as follows:
Then, $A^\prime$ is derived as follows:
\begin{equation}
    A^\prime = (Q\Sigma U^T U \Sigma Q^\top)^l Q\Sigma U^T = Q(\Sigma)^{2l+1} U^T.
\end{equation}

Hence, while $A^\prime$ retains the same singular vectors as $A$, its singular values exhibit an accelerated rate of decay:
\begin{equation}
    \sigma_i(A^\prime) = \sigma_i(A)^{2l+1},
\end{equation}
% a characteristic that is particularly advantageous for low-rank approximation techniques. 
This enhances the efficacy of distinguishing between more significant and less significant singular vectors, thereby improving the overall accuracy of the approximation.
Furthermore, to augment the precision, we utilize a small oversampling parameter $p$, such as $p=5$, to increase the sample size beyond the target rank, thereby providing a buffer that mitigates the risk of omitting significant components. 

Through the S-RSI method, we can efficiently compress the storage of matrix $A$ from $O(mn)$ to $O(k(m+n))$ by retaining only matrices $Q$ and $U^\top$ with time complexity of $O(lmn(k+p))$. Following the randomized SVD algorithm \cite{halko2011finding}, the approximation error bound is:
\begin{equation}\label{equ:error_bound}
\begin{aligned}
    \mathbb{E} \| A - QU^\top\| &\leq \left[ \left( 1 + \sqrt{\frac{k}{p - 1}} \right)^{2l+1} \sigma_{k+1}^{2l+1} \right. \\
     + \frac{e\sqrt{k + p}}{p} &\left. \sqrt{\sum_{j > k} \sigma_j^{2(2l+1)}} \right]^{1/(2l+1)}.
\end{aligned}
\end{equation}
According to Equation \eqref{equ:error_bound}, the approximation error can be reduced by increasing not just $k$, but also $p$ and $l$.
We further demonstrate the efficacy of the S-RSI through empirical comparisons. Specifically, we evaluate its performance in comparison to Adafactor's matrix factorization approach and the SVD, focusing on all second-moment matrices obtained during the training of a GPT-2 345M model using AdamW. Results are shown in Figure \ref{fig:s-rsi}.

\begin{figure*}[t]
    \centering
    \subfigure[Mean approximation error vs. rank.]{
    \includegraphics[width=0.45\textwidth]{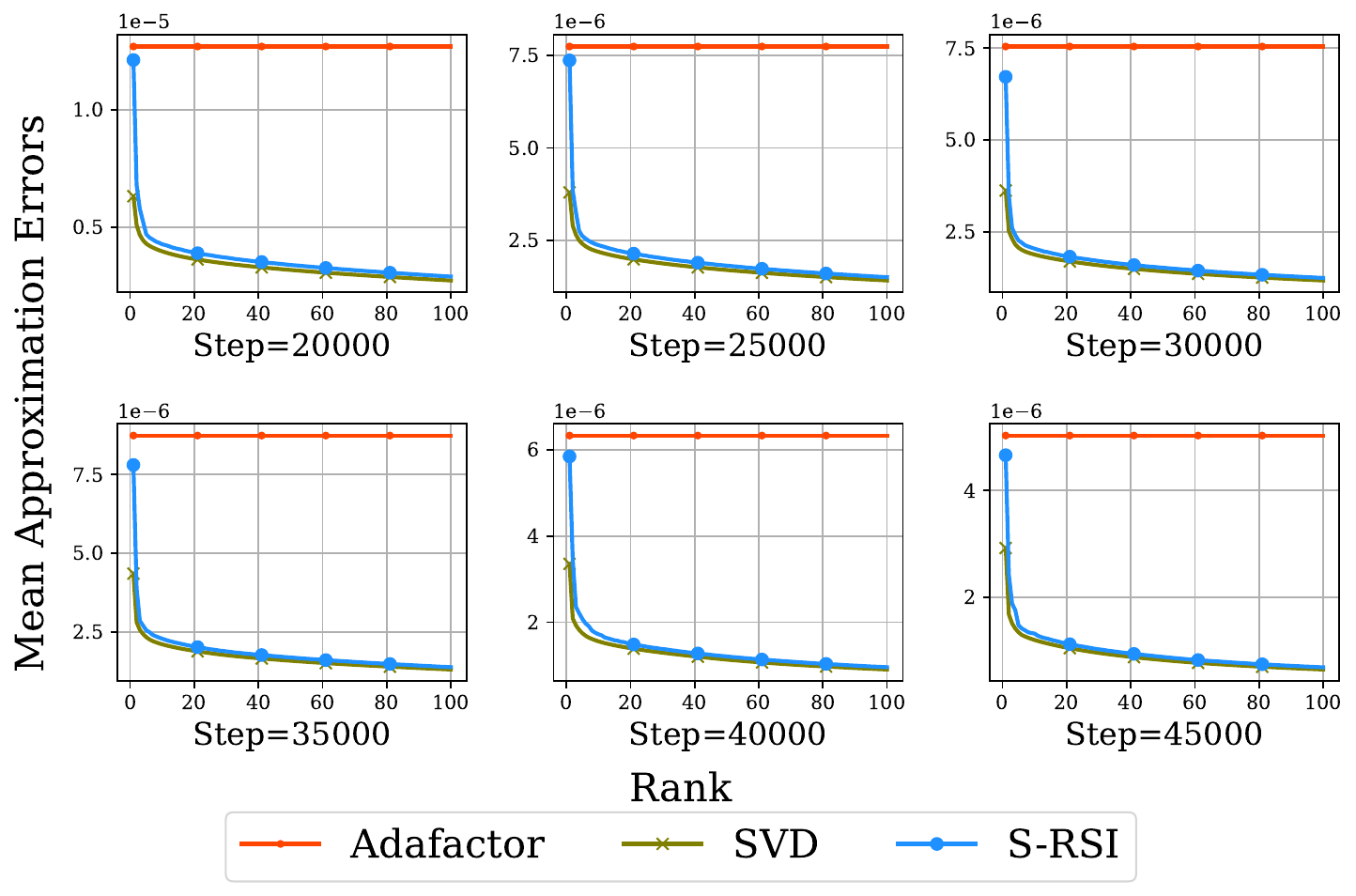} \label{fig:s-rsi-sub1}
    }
    \subfigure[Mean computation time vs. rank.]{
    \includegraphics[width=0.45\textwidth]{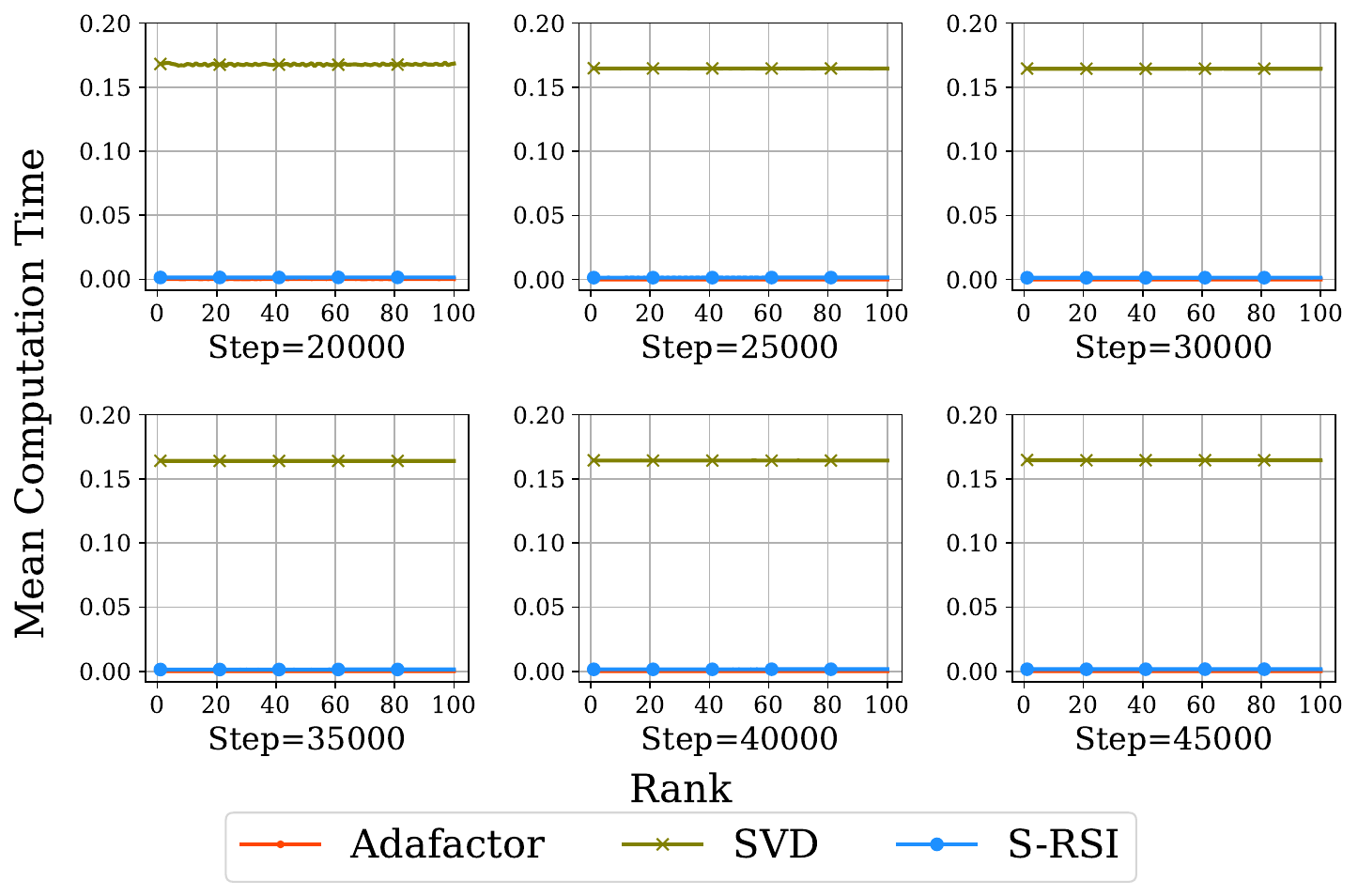}\label{fig:s-rsi-sub2}
    }
    \caption{Comparative analysis of the S-RSI ($l=5$ and $p=5$) against Adafactor and SVD. All methods are applied to the second-moment matrices derived from training a GPT-2 345M model using the AdamW, with results captured at various stages of the training process. 
    % Specific steps are detailed in each corresponding subfigure.
    }
    \label{fig:s-rsi}
\end{figure*}

% Figure \ref{fig:s-rsi} presents the comparative results, showcasing the mean approximation error versus rank alongside the mean computation time versus rank. Since Adafactor's factorization method is a fixed 1-rank approximation, its results remain constant and unaffected by variations in rank. Given that SVD represents the optimal low-rank approximation under the Frobenius norm, its approximation error serves as a benchmark bound. As depicted in Figure \ref{fig:s-rsi}, for both SVD and S-RSI, there is a significant reduction in approximation error with only a slight increase in rank. 
% Consequently, a fixed 1-rank approach may fall short of meeting high precision requirements.
Figure \ref{fig:s-rsi} compares the mean approximation error and computation time of varying ranks, with the S-RSI employing hyperparameters $l=5$ and $p=5$. Adafactor's factorization method employs a fixed 1-rank approximation, leading to consistent results regardless of rank changes. The SVD serves as a benchmark for optimal low-rank approximation.
The SVD and S-RSI show a substantial decrease in approximation error with a minimal increase in rank.
Besides, as the rank increases, the S-RSI approaches the SVD benchmark bound.
In terms of computation time, Adafactor is the most efficient with a consistent duration. The S-RSI not only significantly reduces computation time compared to SVD but also approaches the efficiency of Adafactor.
While the S-RSI shows a linear relationship with rank $k$ theoretically, the actual rate of increase is considerably less than $k$. This can be largely attributed to the multi-core architecture of GPUs and parallel computing.
Consequently, the S-RSI provides an advantageous balance of approximation accuracy and computational efficiency.

% Additionally, S-RSI demonstrates a linear relationship with rank, where the practical linear rate is significantly lower than $k$. This enhanced efficiency can be attributed to the multi-core architecture of GPU. Therefore, the S-RSI offers a favorable balance between approximation accuracy and computational efficiency.

\subsection{Adaptive Rank Selection}
% As demonstrated in Figure \ref{fig:s-rsi}, the selection of an appropriate rank $k$ is crucial for the S-RSI method. A larger $k$ may lead to a computationally demanding process with only marginal gains in precision, whereas a smaller $k$ may substantially reduce the accuracy.
% Therefore, we design an adaptive rank selection mechanism, which dynamically adjusts the rank $k$ in response to each target matrix and according to the progression of training steps.
The results in Figure \ref{fig:s-rsi} underscore the importance of selecting the proper rank $k$ for the S-RSI. Choosing a larger $k$ can result in increased computational demands for minimal precision improvements, while a smaller $k$ risks significantly compromising accuracy. To address this, we develop an adaptive rank selection mechanism, which dynamically adjusts $k$ for each target matrix.

% Specifically, the algorithm initiates with rank $k_t$. Every adaptive interval $\Delta s$, we evaluate the approximation error rate:
The algorithm updates the rank $k_t$ at regular intervals of $\Delta s$ steps. Initially set as $k_\textit{init}$, $k_t$ is used to perform the S-RSI. Then, we evaluate the approximation error rate:
\begin{equation}
    \xi = \frac{\|A-QU^\top\|_F}{\|A\|_F}.
\end{equation}
Should $\xi$ rise above a threshold $\xi_\textit{thresh}$, $k_t$ is adjusted to $k_t + f(\xi)$, ensuring that it does not exceed the maximum rank $k_{\textit{max}}$. We model $f$ as a variant of the sigmoid function including four hyperparameters $\eta, \omega, \phi$ and $\tau$:
\begin{equation}\label{equ:f}
    f(\xi) = \left\lfloor \frac{\eta}{\exp(\omega\xi + \phi) + \tau} \right\rfloor, \ \xi > 0,
\end{equation}
where $\eta > 0$ sets the upper bound of $f$'s range, $\omega < 0$ dictates the rate of change, and $\phi$ and $\tau$ adjust the function's offset. This adheres to the heuristic principle that as $\xi$ increases, the expansion rate of $k_t$ should initially be gradual, then accelerate, and ultimately stabilize near an upper limit.
% This encapsulates the heuristic principle that the expansion rate of \( k \) should initially be gradual, subsequently accelerate, and eventually stabilize near an upper limit.
By Equation \eqref{equ:q_def}, this extension requires sampling $f(\xi)$ additional vectors from a standard Gaussian distribution and then applying QR decomposition again. The revised $k_t$ is then maintained for the next $\Delta s$ iterations. 
The pseudocode for this method is summarized in Algorithm \ref{alg:as-rsi}, which we designate as Adaptive S-RSI (AS-RSI).

\begin{algorithm}[t]
\caption{Adaptive S-RSI}
\label{alg:as-rsi}
\begin{algorithmic}
\STATE \textbf{Inputs:} Target matrix $A \in \mathbb{R}^{m \times n}$, rank $k_{t-1}$, initial rank $k_\textit{init}$, the maximum rank $k_\textit{max}$, integer $l$, integer $p$, threshold $\xi_\textit{thresh}$, step $t$, and adaptive interval $\Delta s$
\IF{$(t \mod \Delta s) = 1$}
    \STATE $k_t \leftarrow k_\textit{init}$
    \REPEAT
        \STATE $Q, U^\top \leftarrow $ \text{S-RSI}($A, k_t, l, p$)
        \STATE $\xi \leftarrow \|A-QU^\top\|_F/\|A\|_F$
        \STATE $k_t \leftarrow \min \{ k_t + f(\xi), k_\textit{max}\}$\COMMENT{Equation \eqref{equ:f}}
        \STATE $p \leftarrow \min\{p, k_\textit{max} - k_t\}$
    \UNTIL{$\xi \leq \xi_\textit{thresh}$}
\ELSE 
\STATE $k_t \leftarrow k_{t-1}$
\STATE $Q, U^\top \leftarrow $ \text{S-RSI}($A, k_t, l, p$)
\ENDIF
\STATE \textbf{return} $Q[:, :k_t], U^\top[:,:k_t]$, $k_t$
\end{algorithmic}
\end{algorithm}

\subsection{Adapprox Algorithm}
\begin{algorithm}[t]
\caption{Adapprox}
\label{alg:adapprox}
\begin{algorithmic}
\STATE \textbf{Inputs:} Initial point $W_0 \in \mathbb{R}^{m \times n}$, $M_0 = {\mathbf0}^{m \times n}$ and $V_0 = \mathbf{0}^{m \times n}$, learning rates $\{\alpha_t\}_{t=1}^T$, moment decay $\beta_1$ and $\beta_2$, small constant $\epsilon$, clipping threshold $d$, initial rank $k_\textit{init}$, the maximum rank $k_\textit{max}$, integer $l$, integer $p$ with $(k_\textit{init}+p) \leq k_\textit{max}$, threshold $\xi_\textit{thresh}$, adaptive interval $\Delta s$, and weight decay rate $\lambda$
\STATE \text{Initialize} $ Q_0 = \mathbf{0}^{m \times k_\textit{init}}, U_0^\top = \mathbf{0}^{n \times k_\textit{init}}$, $k_0 = k_\textit{init}$
\FOR{$t \leftarrow 1, 2, \dots, T$}
    \STATE $G_t \leftarrow \nabla f_t(W_{t-1})$
    \STATE $V_t \leftarrow \beta_2 Q_{t-1}U_{t-1}^\top + (1 - \beta_2) G_t^2$
    % \STATE $ Q_t, U_t^\top, k_t = \text{AS-RSI}(V_t, k_{t-1},  k_\textit{init}, k_\textit{max}, l, p, \epsilon_\textit{thresh}, t, \Delta s)$
    \STATE $ Q_t, U_t^\top, k_t \leftarrow \text{AS-RSI}(V_t, k_{t-1}, k_\textit{init}, k_\textit{max}, l, p, \xi_\textit{thresh},$
    \STATE $\quad\quad\quad\quad\quad\quad t, \Delta s)$
    \STATE $M_t \leftarrow G_t / (\sqrt{V_t} + \epsilon)$
    \STATE $M_t \leftarrow M_t / \max(1, \text{RMS}(M_t) / d)$
    \IF{$\beta_1 > 0$}
        \STATE $M_t \leftarrow \beta_1 M_{t-1} + (1 - \beta_1) M_t$
    \ENDIF
    \STATE $W_t \leftarrow W_{t-1} - \alpha_t (M_t + \lambda W_{t-1})$
\ENDFOR
\end{algorithmic}
\end{algorithm}
The integration of the proposed methodologies culminates in the Adapprox algorithm defined in Algorithm \ref{alg:adapprox}.
At each step, $f_t$ represents a stochastic realization of the objective function $f$, exemplified by the loss function computed using a randomly selected mini-batch of data. We then compute the gradient $G_t$ relative to the previous parameters and the exponential running averages of the second moment $V_{t}$. 
Note that $V_{t-1}$ is reconstructed from $Q_{t-1}U_{t-1}^\top$, while $V_{t}$ is factored using the AS-RSI. 
% Specifically, AS-RSI selects the rank $k_0$ within the range from $k_{\textit{init}}$ to  $k_{\textit{max}}$. This selected rank is then utilized as the input for S-RSI during the ensuing interval $\Delta s$.
Subsequently, we calculate the update $M_t = G_t/(\sqrt{V_t} + \epsilon)$ and incorporate the update clipping mechanism as proposed in \cite{shazeer2018adafactor} to mitigate excessively large updates:
\begin{equation*}
    M_t \leftarrow \frac{M_t}{\max(1, {\text{RMS}(M_t)}/{d})},\ \text{RMS}(M_t) = \frac{\|M_t\|_F}{\sqrt{mn}},
\end{equation*}
where $d$ is the clipping threshold.
We also offer the option to omit the first moment, depending on whether $\beta_1$ is set to zero. Finally, parameter updates are executed in a decoupled weight decay fashion, as delineated in Equation \eqref{equ:weight_decay}.

The proposed Adapprox, while rooted in Adam, diverges in three key aspects: (1) it omits the bias correction steps; (2) it integrates an update clipping mechanism that enhances performance (as detailed in Appendix~\ref{app:clip}); 
and (3) it modifies the first moment accumulator, replacing the running average of the gradient with the running average of the update.

\subsection{Cosine-Similarity Guidance Strategy}
% Taking into account the inherent errors in low-rank approximation, 
Inspired by CAME, we incorporate an optional mechanism aimed at enhancing the update process and expediting the training phase when $\beta_1 > 0$. Specifically, this involves a heuristic method that utilizes the cosine similarity between the current update and the running average of updates to assess the confidence level of the approximation.

To elaborate, for the given current update
\begin{equation}
\hat{M}_t \leftarrow \frac{G_t}{\sqrt{V_t} + \epsilon}, \
\hat{M}_t \leftarrow \frac{\hat{M}_t}{\max(1, {\text{RMS}(\hat{M}_t)}/{d})},
\end{equation}
and the running average of updates
\begin{equation}
M_t \leftarrow \beta_1 M_{t-1} + (1 - \beta_1) \hat{M}_t,
\end{equation}
we calculate the cosine similarity measure
\begin{equation}
\theta_{\textit{cos}} = \frac{\sum_{i,j}(\hat{M}_t \odot M_t)_{ij}}{\|\hat{M}_t\|_F \| M_t\|_F} \in [-1, 1],
\end{equation}
where $\odot$ denotes element-wise multiplication.

When $\theta_{\textit{cos}}$ approaches 1, it indicates a better alignment between the directions of $\hat{M}_t$ and $M_t$. Conversely, a $\theta_{\textit{cos}}$ value close to $-1$ suggests poor alignment. In the case where $\theta_{\textit{cos}} = 0$, $\hat{M}_t$ and $M_t$ are orthogonal to each other.
To modulate the update based on $\theta_{\textit{cos}}$, we introduce a factor that either penalizes or accelerates the update accordingly:
\begin{equation}
    M_t \leftarrow \frac{M_t}{1 - \theta_{\textit{cos}} + \epsilon},
\end{equation}
which adjusts $M_t$ inversely with respect to $\theta_{\textit{cos}}$, enhancing the update for alignment and penalizing it for divergence.

It is noteworthy that our proposed cosine-similarity guidance strategy is effective exclusively when $\beta_1 > 0$, which is consistent with CAME. However, our strategy is designed to utilize only the information available at the current step, thereby obviating the need for extra memory.

\section{Experiments}
In this section, we present a comparative analysis of Adapprox against existing counterparts in pretraining GPT-2 \cite{radford2019language} and in subsequent downstream tasks.
\newcolumntype{Y}{>{\centering\arraybackslash}X}
\subsection{Setup}
We investigate two GPT-2 configurations: 117M and 345M, as outlined in Table \ref{tab:config}. Our pretraining experiments utilize The Pile dataset \cite{gao2020pile} and the SentencePiece tokenizer \cite{kudo2018sentencepiece}. We evaluate the pretrained models on several downstream tasks, including SQuAD~v1.1 \cite{rajpurkar2016squad}, CoLA \cite{warstadt2018neural}, MRPC \cite{dolan2005automatically}, SST-2 \cite{socher2013recursive}, and MNLI-m \cite{williams2017broad}. 
% This expansion is aimed at demonstrating the effectiveness of pre-trained GPT-2 models optimized with Adapprox, as well as evaluating Adapprox's fine-tuning capabilities on various downstream tasks.
Our primary baselines include AdamW \cite{loshchilov2018decoupled}, Adafactor \cite{shazeer2018adafactor}, and CAME \cite{luo2023came}. We have implemented our optimization algorithm using the PyTorch framework \cite{paszke2019pytorch}. Additionally, the pretraining of GPT-2 is conducted utilizing the Megatron-LM framework \cite{shoeybi2019megatron} and eight NVIDIA Tesla V-100 GPUs. 
% ensuring consistency by conducting all tests within the same computational environment.

% For GPT-2 pertaining, we consistently set the parameters $\beta_1$ and $\beta_2$ to 0.9 and 0.999, respectively, and apply a weight decay rate of 0.1 across all algorithms under comparison.
% Other parameters for Adapprox are set as follows: $a=200$, $b=-10$, $c=-2.5$, $d=-9$, $\epsilon = 1 \times 10^{-8}$, $d = 1$, $k_\textit{init}=1$, $k_\textit{max}=0.25\min\{m, n\}$, $l=5$, $p=5$, $\epsilon_\textit{thresh}=0.01$, and $\Delta s = 10$. For Adafactor and CAME, we retain the remaining parameters according to their default implementations.
For GPT-2 pretraining, we set $\beta_1$ and $\beta_2$ at 0.9 and 0.999, respectively, and maintain a consistent weight decay rate of 0.1 for all compared algorithms. Adapprox's additional parameters are specified as follows: $\epsilon = 1 \times 10^{-8}$, $d = 1$, $k_\textit{init}=1$, $k_\textit{max}=0.25\min\{m, n\}$, $l=5$, $p=5$, $\xi_\textit{thresh}=0.01$, $\Delta s = 10$, and $\eta=200$, $\omega=-10$, $\phi=-2.5$, $\tau=-9$. For Adafactor and CAME, the other parameters are set to their respective default values. We adopt a linear warmup strategy followed by a cosine-style learning rate decay, both integrated within the Megatron-LM framework. To guarantee fair comparisons, all evaluated optimizers use uniform training parameters for each model, selected through empirical testing and established best practices. Specifically, the batch size, number of training iterations, number of warmup iterations, peak learning rate, and minimum learning rate are set as follows: for GPT-2 117M, they are 128, 100K, 1K, \(3\times 10^{-4}\), and \(5\times 10^{-5}\); and for GPT-2 345M, 128, 100K, 1K, \(3\times 10^{-4}\), and \(3\times 10^{-5}\), respectively. 

For downstream tasks, we fine-tune GPT-2 models pretrained with each evaluated optimizer for three epochs, adjusting learning rates individually per task. 
Besides, cosine-similarity guidance is not employed in the fine-tuning process of Adapprox.
%to prioritize precise adjustments based on the pretrained model.

\begin{table}[t]
\caption{Model configurations.}
\label{tab:config}
\vskip 0.15in
\begin{center}
\begin{small}
\begin{sc}
\begin{tabular}{ccccc}
\toprule
\multirow{2}{*}{\textsc{Size}} & \multirow{2}{*}{\textsc{Layers}} & \multirow{2}{*}{\textsc{Hidden}} & \multirow{2}{*}{\textsc{Heads}} & Sequence\\
&&&&Length\\
\midrule
117M    & 12 & 768 & 12 & 1024\\
345M    & 24 & 1024 & 16& 1024\\
%1.3B    & 32 & 1792 & 16& 1024 \\
\bottomrule
\end{tabular}
\end{sc}
\end{small}
\end{center}
\vskip -0.1in
\end{table}
\subsection{Memory Usage Comparison}
\begin{figure*}[t]
    \centering
    \subfigure[117M, validation loss vs. iteration.]{
    \includegraphics[width=0.35\textwidth]{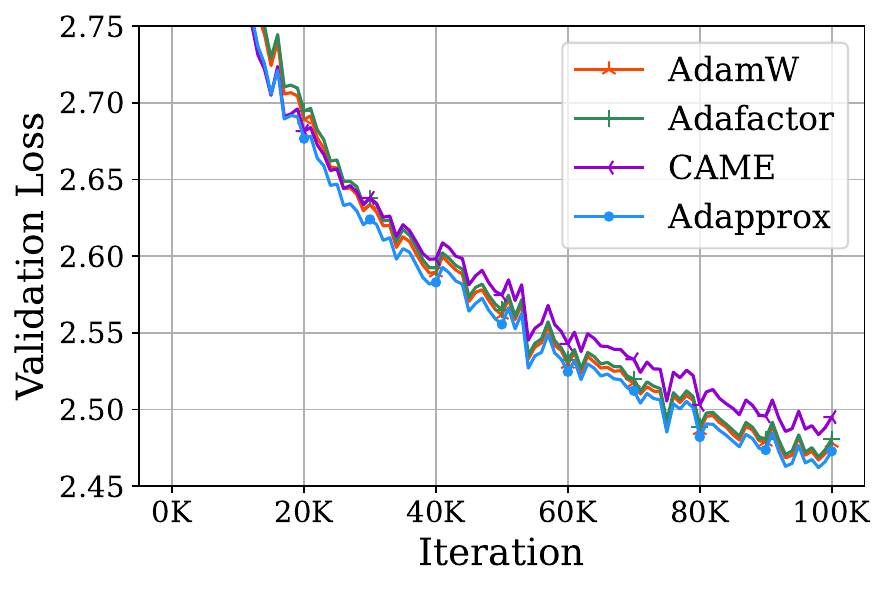} \label{fig:loss-sub1}
    }\hspace{0.1\textwidth}
    \subfigure[345M, validation loss vs. iteration.]{
    \includegraphics[width=0.35\textwidth]{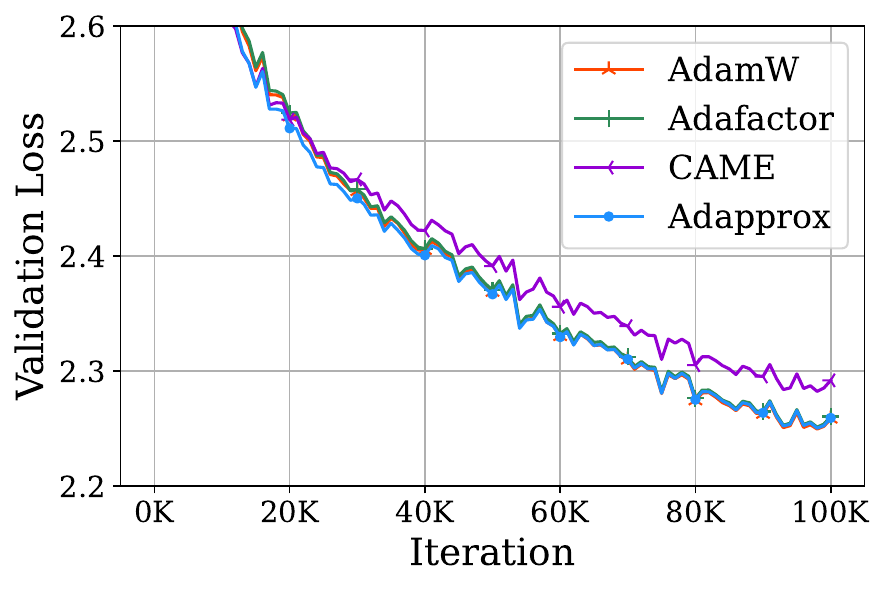}\label{fig:loss-sub2}
    }
    % \subfigure[1.3B, validation loss vs. iteration.]{
    % \includegraphics[width=0.3\textwidth]{icml2024/figs/loss_compare_1.3B_final.pdf}\label{fig:loss-sub3}
    % } 
    \\
    \subfigure[117M, perplexity vs. iteration.]{
    \includegraphics[width=0.35\textwidth]{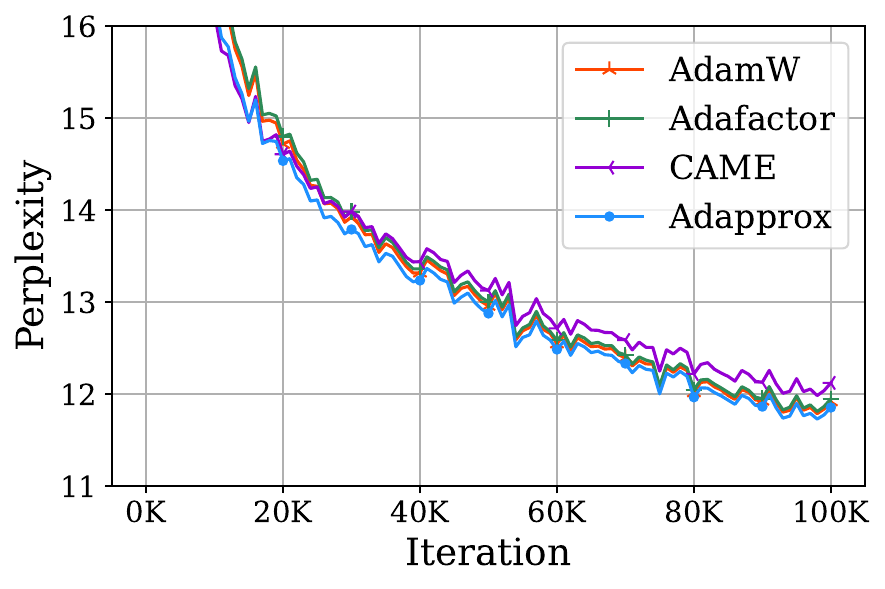} \label{fig:loss-sub4}
    }\hspace{0.1\textwidth}
    \subfigure[345M, perplexity vs. iteration.]{
    \includegraphics[width=0.35\textwidth]{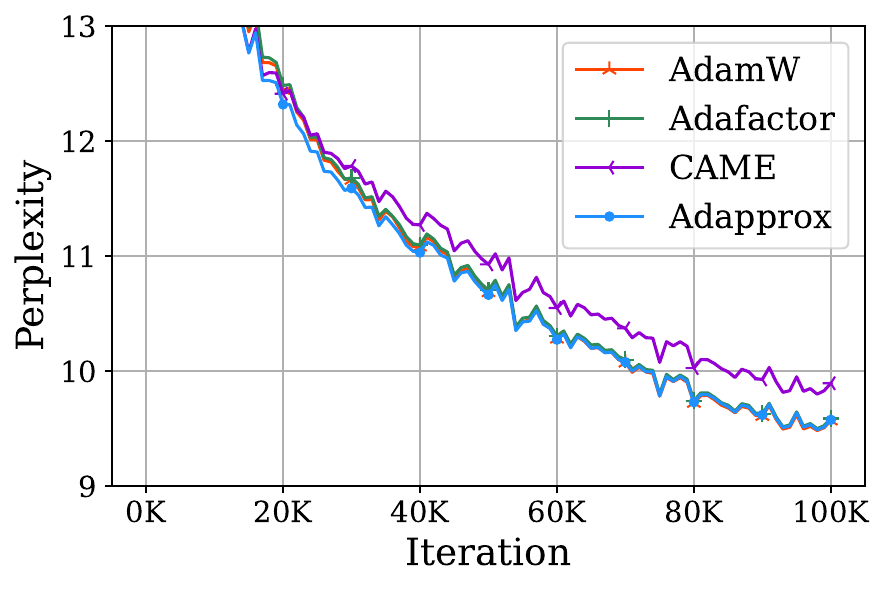}\label{fig:loss-sub5}
    }
    % \subfigure[1.3B, perplexity vs. iteration.]{
    % \includegraphics[width=0.3\textwidth]{icml2024/figs/ppl_compare_1.3B_final.pdf}\label{fig:loss-sub6}
    % }
    \caption{Comparative analysis of Adapprox against AdamW, Adafactor, and CAME on pretraining GPT-2 models.}
    \label{fig:loss}
\end{figure*}

\begin{table}[t]
\caption{Quantitative memory usage (MB) comparison.}
\label{tab:mem}
\vskip 0.15in
\begin{center}
\begin{small}
\begin{tabular}{cccc}
\toprule
$\beta_1$ & \textsc{Method} & \textsc{GPT-2 117M} & \textsc{GPT-2 345M} \\
\midrule
\multirow{5}{*}{0.9}    & AdamW & 949.7 (100.0\%)  & 2707.5 (100.0\%) \\
                        & Adafactor & 476.1 (50.1\%) & 1356.7 (50.1\%) \\
                        & CAME & 476.8 (50.2\%) & 1358.4 (50.2\%)\\
                        & Adapprox ($k_\textit{init}$) & 476.1  (50.1\%) &1356.7  (50.1\%) \\
                        & Adapprox ($k_\textit{max}$) & 622.0 (65.5\%) & 1791.1  (65.5\%)\\

\midrule
\multirow{5}{*}{0.0}    & AdamW & 949.7  (100.0\%) &2707.5 (100.0\%)\\
                        & Adafactor & 1.2 (0.1\%) & 2.9 (0.1\%) \\
                        & CAME & $-$  & $-$ \\
                        & Adapprox ($k_\textit{init}$) & 1.2 (0.1\%) & 2.9 (0.1\%)\\
                        & Adapprox ($k_\textit{max}$) &147.2 (15.5\%) & 437.4 (16.2\%) \\ 
\bottomrule
\end{tabular}
\end{small}
\end{center}
\vskip -0.1in
\end{table}

\begin{table*}[!ht]
\caption{Results of fine-tuning performance on downstream tasks.}
\label{tab:down}
\vskip 0.15in
\begin{center}
\begin{small}
\begin{tabularx}{\linewidth}{YYYYYYYY}
\toprule
\multirow{2}{*}{\textsc{Model}}  & \multirow{2}{*}{\textsc{Method}} & SQuAD v1.1 & CoLA & MRPC & SST-2 & MNLI-m &\multirow{2}{*}{\textsc{Average}} \\
& & (F1) & (Acc) & (Acc) & (Acc) & (Acc) &\\
\midrule
\multirow{5}{*}{GPT-2 117M}    & AdamW & 70.35 & 74.69 &\textbf{82.35}&89.33&79.57 & 79.26\\
                        & Adafactor & 72.57 &75.55 &79.41&89.33&80.13 &79.40\\
                        & CAME & 70.60 & 69.31 &64.95&82.22&76.37 & 72.69 \\
                        & Adapprox & \textbf{72.89}  &\textbf{75.83} &80.63&\textbf{90.94}&\textbf{80.15} & \textbf{80.09} \\

\midrule
\multirow{4}{*}{GPT-2 345M}    & AdamW & 75.06 &\textbf{79.39} &82.84&\textbf{91.17}&82.33 & 82.16\\
                        & Adafactor & 76.23   & 78.90 &\textbf{83.33}&91.06 & 82.62 & 82.43\\
                        & CAME & 75.50  & 69.22 &71.81&82.22 & 81.72 & 76.09\\
                        & Adapprox & \textbf{76.68} & 79.00 &\textbf{83.33}&\textbf{91.17}&\textbf{82.66} & \textbf{82.57}\\
\bottomrule
\end{tabularx}
\end{small}
\end{center}
\vskip -0.1in
\end{table*}

We assess the memory footprint associated with the state of each optimizer. Adapprox utilizes an adaptive low-rank approximation, and we report its memory usage based on predetermined $k_{\textit{init}}$ and $k_{\textit{max}}$ values. The actual memory usage falls between these boundaries. The results are listed in Table \ref{tab:mem}. While Adapprox does increase memory usage compared to Adafactor and CAME, it allows for a flexible trade-off between memory efficiency and accuracy through the adjustment of $k_{\textit{init}}$ and $k_{\textit{max}}$. In comparison to AdamW and with $\beta_1 = 0.9$, Adapprox achieves memory reductions of 34.5\% to 49.9\% for GPT-2 117M and 33.8\% to 49.9\% for GPT-2 345M. With $\beta_1 = 0$, CAME becomes non-viable, while AdamW still allocates memory for the first moment. In contrast, both Adafactor and Adapprox completely forego first moment memory allocation. In such scenarios, Adapprox yields significant memory savings: 84.5\% to 99.9\% for GPT-2 117M and 83.8\% to 99.9\% for GPT-2 345M.

\subsection{GPT-2 Training}
Figure \ref{fig:loss} presents the empirical comparison of Adapprox with AdamW, Adafactor, and CAME in the pretraining of GPT-2 models. The top row shows the validation loss across the 117M and 345M GPT-2 models, respectively, from left to right. The bottom row displays the validation perplexity for these models, arranged in the same sequence.
% Our analysis suggests that Adapprox achieves better validation loss and perplexity compared to Adafactor and CAME. 
Relative to AdamW, Adapprox generally exhibits better validation loss and perplexity on the GPT-2 117M model, while delivering comparable results on the GPT-2 345M. 
Adafactor consistently underperforms compared to Adapprox, although the performance difference on the GPT-2 345M is relatively modest. 
While CAME initially shows lower validation loss and perplexity, it tends to converge to suboptimal outcomes over time. These results indicate that Adapprox effectively balances accuracy with memory usage and potentially offers faster convergence and superior performance compared to Adafactor and CAME.

\subsection{Downstream Tasks}
%We evaluate the downstream task performance of GPT-2 110M and 345M models, ensuring that each model is both pretrained and fine-tuned with its corresponding optimizer. Results for the GPT-2 1.3B are not presented here due to space constraints.
We evaluate the downstream task performance of GPT-2 110M and 345M models, each pretrained and fine-tuned with its corresponding optimizer. The empirical results are presented in Table~\ref{tab:down}.
Our experimental findings underscore the superiority of Adapprox, as evidenced by better performance of GPT-2 models trained and fine-tuned using it, compared to those using Adafactor and CAME. Besides, Adapprox not only achieves performance comparable to AdamW but also surpasses it in certain tasks.
Additionally, we observe that Adapprox exhibits stable performance across various learning rates, a conclusion drawn from employing grid search to determine optimal learning rates for different downstream tasks. 
In contrast, CAME shows sensitivity to learning rate selection, frequently underperforming relative to other optimizers in these sceniors. 
For detailed evidence, refer to Appendix \ref{app:lr}.

\section{Discussion}
While Adapprox and Adafactor achieve significant memory savings through the low-rank approximation of the second moment and by omitting the first moment (seen in Table \ref{tab:mem}), our experiments demonstrate that the absence of the first moment has a notable impact on convergence speed and overall performance (see Appendix \ref{app:beta1}). Consequently, we recommend retaining the first moment unless memory constraints are extremely prohibitive. Looking towards future research, investigating techniques to compress the first moment could be a promising direction. Furthermore, it is crucial to note that our approach is compatible with other memory optimization techniques such as quantization and recomputation. This compatibility opens the door for potential synergistic enhancements in optimizer design.

\section{Conclusion}
In this paper, we introduce Adapprox, an innovative optimizer designed to address the memory consumption challenges in training large-scale models. By applying randomized low-rank matrix approximation to the second moment of Adam, Adapprox achieves significant memory reduction with minimal impact on model accuracy. This solution is further optimized with an adaptive rank selection mechanism and an optional cosine similarity guidance strategy, enhancing both stability and convergence speed. Our empirical evaluations, encompassing GPT-2 117M and 345M models and their downstream tasks, demonstrate Adapprox's efficacy. It delivers considerable memory savings, ranging from 33.8\% to 49.9\% over AdamW with the first moment intact, and between 83.8\% to 99.9\% when the first moment is excluded. Despite a slight trade-off in memory efficiency, Adapprox outperforms competitors like Adafactor and CAME across several key metrics, including validation loss and perplexity in pretraining, and F1 score and accuracy in downstream tasks. These results establish Adapprox as a balanced approach for memory-efficient training, effectively harmonizing efficiency with minimal accuracy trade-offs.

\section*{Impact Statements}
This paper presents work whose goal is to advance the field of Machine Learning. There are many potential societal consequences of our work, none which we feel must be specifically highlighted here.

\bibliography{example_paper}
\bibliographystyle{icml2024}

%%%%%%%%%%%%%%%%%%%%%%%%%%%%%%%%%%%%%%%%%%%%%%%%%%%%%%%%%%%%%%%%%%%%%%%%%%%%%%%
%%%%%%%%%%%%%%%%%%%%%%%%%%%%%%%%%%%%%%%%%%%%%%%%%%%%%%%%%%%%%%%%%%%%%%%%%%%%%%%
% APPENDIX
%%%%%%%%%%%%%%%%%%%%%%%%%%%%%%%%%%%%%%%%%%%%%%%%%%%%%%%%%%%%%%%%%%%%%%%%%%%%%%%
%%%%%%%%%%%%%%%%%%%%%%%%%%%%%%%%%%%%%%%%%%%%%%%%%%%%%%%%%%%%%%%%%%%%%%%%%%%%%%%
\newpage
\appendix
\onecolumn
\section{Ablation Experiment on Clipping Mechanism}\label{app:clip}
\begin{figure}[ht]
\vskip 0.2in
\begin{center}
\centerline{\includegraphics[width=0.5\columnwidth]{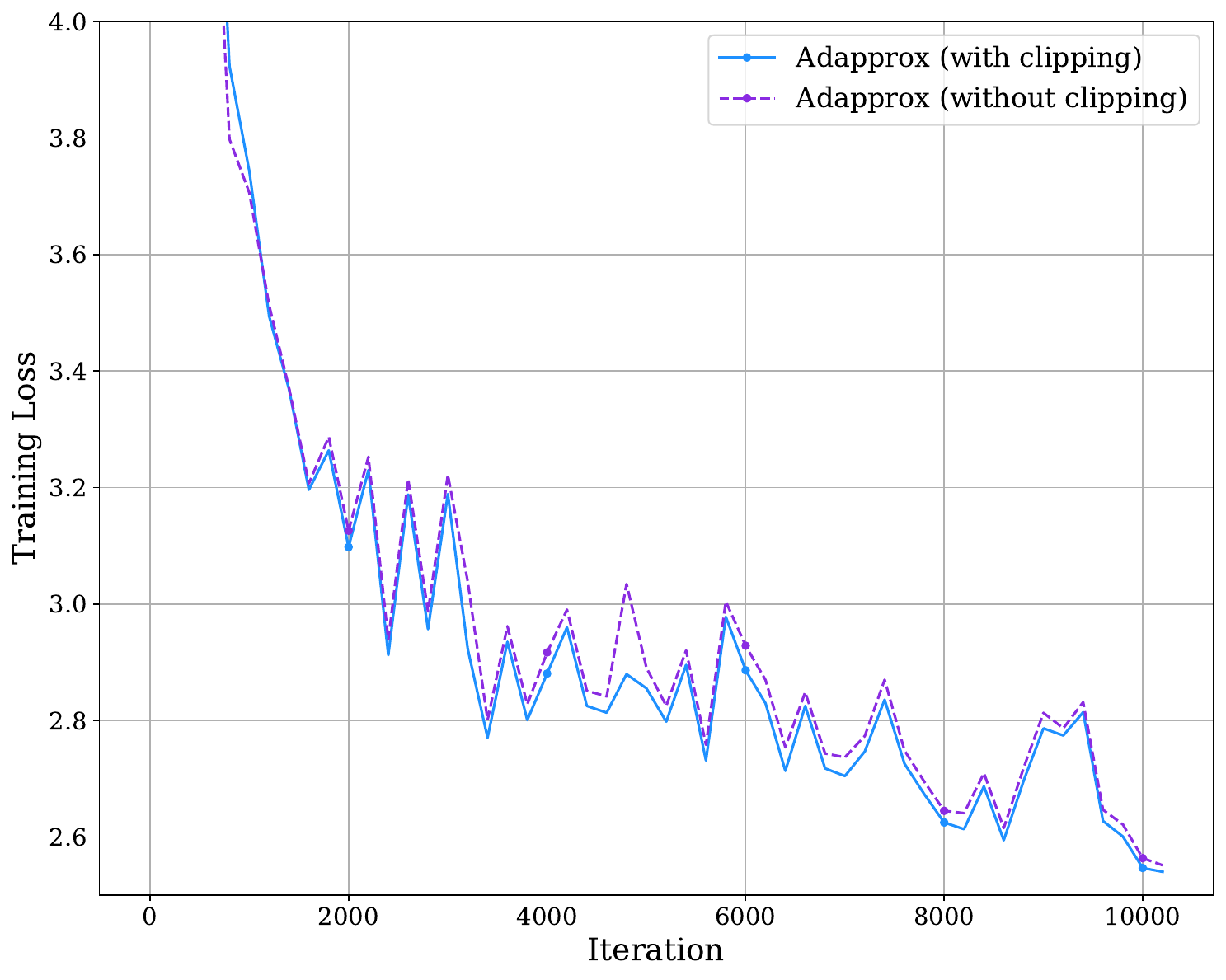}}
\caption{Comparative analysis of training loss for the GPT-2 345M model utilizing Adapprox with and without the clipping mechanism.}
\label{fig:clip}
\end{center}
\vskip -0.2in
\end{figure}
Figure \ref{fig:clip} provides a comparative analysis of training loss convergence for the GPT-2 345M model utilizing Adapprox with and without the implementation of the clipping mechanism. The results indicate that the inclusion of a clipping mechanism enhances Adapprox's performance, resulting in lower training losses at equivalent iterations.
\section{Learning Rate Sensitivity Analysis}\label{app:lr}
\begin{figure}[ht]
\vskip 0.2in
\begin{center}
\centerline{\includegraphics[width=0.6\columnwidth]{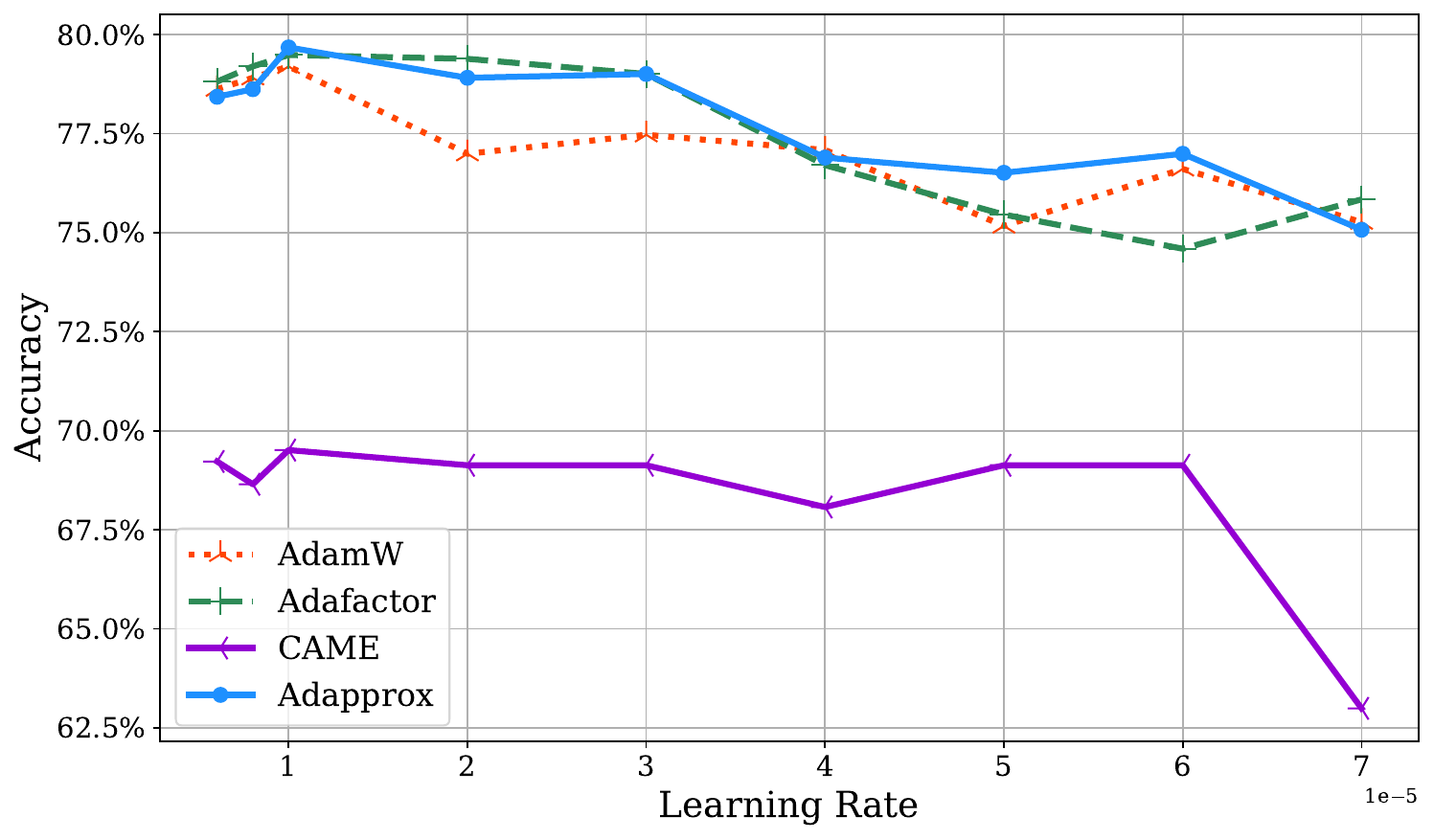}}
\caption{Accuracy of the AdamW-pretrained GPT-2 345M model fine-tuned with compared optimizers on the CoLA task across different learning rates.}
\label{fig:lr-analysis}
\end{center}
\vskip -0.2in
\end{figure}
In our experiments, Adapprox's performance demonstrates stability across a range of learning rates. This is illustrated in Figure \ref{fig:lr-analysis}, which compares the accuracy of various optimizers in fine-tuning the AdamW-pretrained GPT-2 345M model on the CoLA task under different learning rate conditions. Notably, Adapprox shows minimal fluctuations, underlining its enhanced stability and lower sensitivity to changes in learning rates.
% \Furthermore, CAME, in contrast, consistently yields suboptimal results compared to its counterparts. 

\section{First Moment Efficacy Analysis}\label{app:beta1}
\begin{figure}[ht]
\vskip 0.2in
\begin{center}
\centerline{\includegraphics[width=0.5\columnwidth]{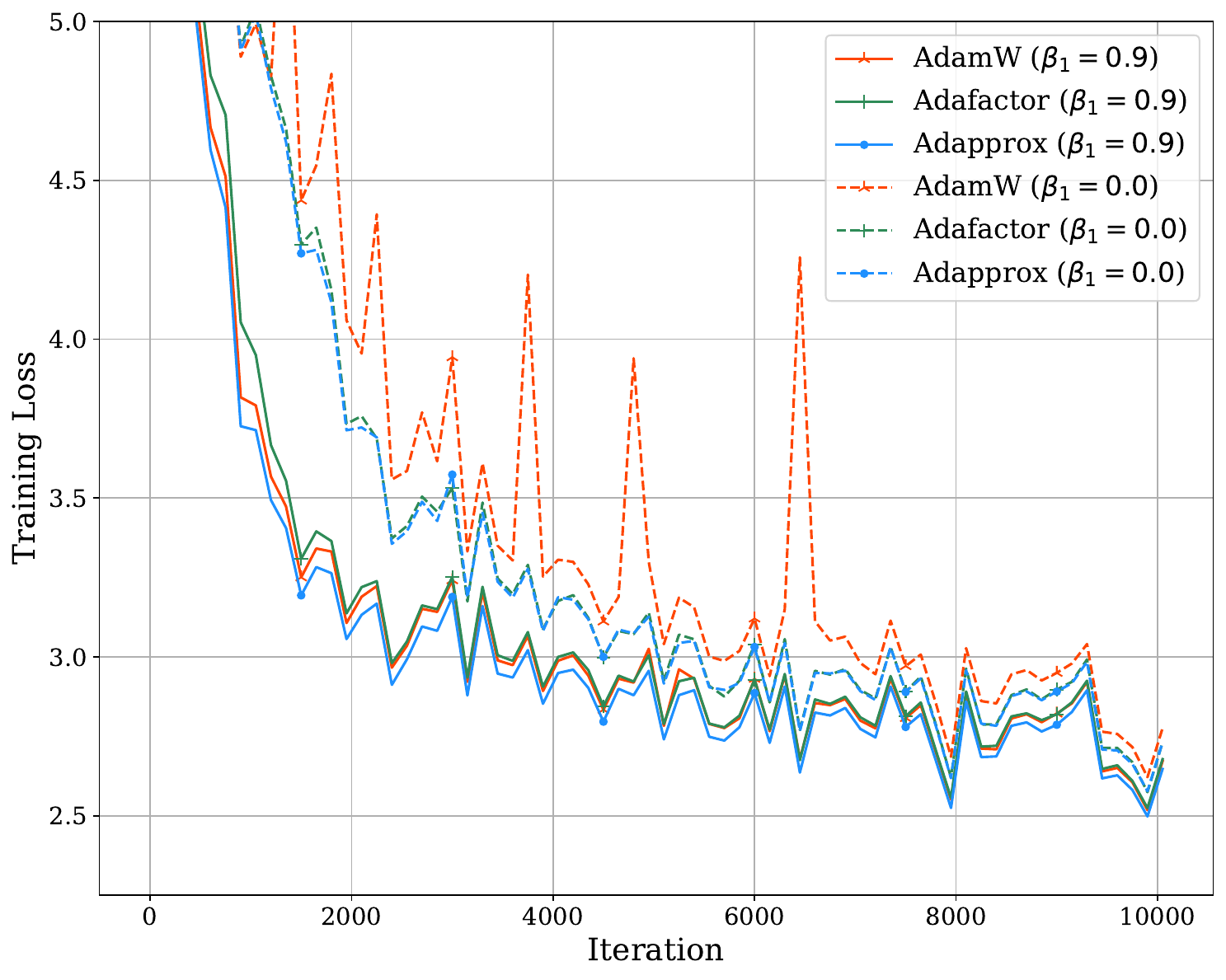}}
\caption{Training loss vs. iteration for AdamW, Adafactor, and Adapprox optimizers, comparing scenarios with and without the first moment. }
\label{fig:beta1-analysis}
\end{center}
\vskip -0.2in
\end{figure}
Figure \ref{fig:beta1-analysis} demonstrates that incorporating the first moment significantly accelerates the convergence process, evidenced by achieving lower training losses at the same iteration, for each optimizer examined, including AdamW, Adafactor, and Adapprox. CAME is omitted from this analysis due to its incompatibility with $\beta_1 = 0$. Furthermore, while AdamW exhibits instability without the first moment, Adafactor and Adapprox mitigate this through the use of a clipping mechanism, effectively reducing large, unexpected updates and enhancing stability.
% This is evidenced by a marked decrease in training loss across iterations when the first moment is active, as opposed to when it is not included.

%%%%%%%%%%%%%%%%%%%%%%%%%%%%%%%%%%%%%%%%%%%%%%%%%%%%%%%%%%%%%%%%%%%%%%%%%%%%%%%
%%%%%%%%%%%%%%%%%%%%%%%%%%%%%%%%%%%%%%%%%%%%%%%%%%%%%%%%%%%%%%%%%%%%%%%%%%%%%%%

\end{document}